\documentclass[]{article}
\usepackage[preprint, nonatbib]{neurips_2019}
\usepackage{graphicx}
\usepackage[T1]{fontenc}
\usepackage[utf8]{inputenc}
\usepackage[english]{babel}
\usepackage{amsmath,amsthm,amssymb,amsfonts}
\usepackage{lmodern}
\usepackage{subcaption}
\usepackage{geometry}
\usepackage{graphicx}
\usepackage{algorithm}
\usepackage{algpseudocode}
\usepackage{empheq}

 \usepackage{empheq}
 \usepackage{color}
\definecolor{darkgreen}{rgb}{0.00,0.5,0.00}
\usepackage{hyperref}

\usepackage{framed}
\usepackage{enumerate}
\usepackage{import}

\newtheorem{theorem}{Theorem}
\theoremstyle{definition}
\newtheorem{lemma}{Lemma}

\newcommand{\R}{\mathbb R}
\newcommand{\eqdef}{\stackrel{\text{def}}{=}}
\renewcommand{\o}{\omega}
\newcommand{\1}{\textbf{1}}
\newcommand{\KL}{\mathcal{KL}}
\DeclareMathOperator{\prox}{prox}

\newcommand{\grad}{\nabla \omega} % gradient of distance generated function

\DeclareMathOperator{\argmin}{argmin}
\DeclareMathOperator{\dom}{dom}
\DeclareMathOperator{\diag}{diag}
\def\<#1,#2>{\left\langle #1,#2\right\rangle}

\usepackage{xcolor, color, colortbl}
\definecolor{myblue}{HTML}{D2E4FC}
\newcommand*\mybluebox[1]{\colorbox{myblue}{\hspace{1em}#1\hspace{1em}}}

\usepackage[colorinlistoftodos,prependcaption,backgroundcolor=black!5!white,bordercolor=red]{todonotes}
\graphicspath{ {plots/} }

\title{Sinkhorn Algorithm as a Special Case\\ of Stochastic Mirror Descent}
\author{Konstantin Mishchenko \\ KAUST\thanks{King Abdullah University of Science and Technology, Thuwal, Saudi Arabia}}
\begin{document}
\maketitle

\begin{abstract}
	We present a new perspective on the celebrated Sinkhorn algorithm by showing that is a special case of incremental/stochastic mirror descent. In order to see this, one should simply plug Kullback-Leibler divergence in both mirror map and the objective function. Since the problem has unbounded domain, the objective function is neither smooth nor it has bounded gradients. However, one can still approach the problem using the notion of relative smoothness, obtaining that the stochastic objective is 1-relative smooth. The discovered equivalence allows us to propose 1) new methods for optimal transport, 2) an extension of Sinkhorn algorithm beyond two constraints.
\end{abstract}
\section{Introduction}
The optimal transport problem consists of finding the best coupling of distributions according to some cost matrix. This problems arises in many machine learning applications, e.g.\ see~\cite{peyre2019computational, dvurechensky2018computational, lin2019efficient, lin2019acceleration} for literature overviews.

The entropy-regularized optimal transport problem proposed by Cuturi~\cite{cuturi2013sinkhorn} is formulated as
\begin{align}
	\min_{X\in R^{N\times N}} \<C, X> + \gamma\sum_{i, j=1}^N X_{ij} \log X_{ij} \qquad \text{s.t.}~  X\1=p, X^\top \1 =q,  \label{eq:pb_reg}
\end{align}
where $\gamma>0$, $C\in R^{N\times N}$ and $p,q\in\R^N$ are given, $\1\in\R^N$ denotes the vector of all 1s, typically $C$ has non-negative entries and $p,q\in \R^N$ are from probability simplex. The problem arises as regularization of the linear programming problem proposed by Kantorovich~\cite{kantorovich1942translation},
\begin{align}
	\min_{X\in R^{N\times N}} \<C, X> \quad \text{s.t.}\quad  X\1=p, X^\top \1 =q, X_{ij}\ge 0~ \forall i,j. \label{eq:original}
\end{align}
Usually problem~\eqref{eq:pb_reg} is solved using Sinkhorn iterations~\cite{sinkhorn1967diagonal, knight2008sinkhorn, kalantari2008complexity}, which iteratively normalize all rows and all columns until convergence. It was first applied to optimal transport by Cuturi~\cite{cuturi2013sinkhorn} and since then is considered as the standard approach.

To provide a simple explanation of why Sinkhorn algorithm works, one can mention that it is an instance of iterative Bregman projections~\cite{benamou2015iterative}. Indeed, problem~\eqref{eq:pb_reg} is effectively the problem of finding a projection in Kullback-Leibler distances, i.e.
\begin{align}
	\min_{X\in \R^{N\times N}} \KL(X||X^0) \quad \text{s.t.}\quad X \1 =p, X^\top \1 =q, \label{eq:pb_kl_proj}
\end{align} 
where $X^0\eqdef \exp\left(-\frac{C}{\gamma}\right)$ and $\KL(X||X^0)\eqdef \sum_{i,j=1}^N (X_i\log \frac{X_i}{X_i^0}-X_i+X_i^0)$ . Since the constraints are linear, it is sufficient~\cite{benamou2015iterative} to project onto them incrementally starting from $X^0$ and the method will converge to the solution of~\eqref{eq:pb_kl_proj}. 

However, there a few problems with this interpretation. First, Bregman projection in general does not posses a closed-form solution even in the case of linear subspaces~\cite{dhillon2007matrix}, so we can not easily generalize and use Sinkhorn algorithm in other settings. Second, the method of Bregman projections is not well-understood and there is no known accelerated version (in the sense of Nesterov~\cite{nesterov1983method}) of it, while mirror descent has been studied for decades from various perspectives~\cite{nemirovsky1983problem, beck2003mirror, auslender2006interior, beck2017first, lu2018relatively, hanzely2018accelerated}.  

%For completeness, let us also mention that a number of methods has been proposed that use minmax reformulation of either~\eqref{eq:pb_reg} or~\eqref{eq:original}. For instance, one can use extragradient~\cite{korpelevichextrapolation, nemirovski2004prox, mishchenko2019revisiting} method~\cite{jambulapati2019direct} or a primal-dual method~\cite{dvurechensky2018computational, guo2019accelerated}.

\subsection{Mirror descent}
The goal of mirror descent is to find a better distance than Euclidian in order to get a better behaviour of gradient descent. Let $\o:\R^d\to \R$ be a proper closed and convex function. We define Bregman divergence associated with $\o$ for any $x$ and $y$ as 
\begin{align*}
	D_\o(x, y)
	= \o(x) - \o(y) - \<\grad(y), x - y>.
\end{align*}
In the unconstrained case, which is of main interest to us, mirror descent can be written in two equivalent ways. If we are given $x^k\in \dom(\o)$, then the next iterated is produced by either of the following two equations,
\begin{align*}
	x^{k+1} &=\argmin_{x}\left\{\eta\<\nabla f(x^k), x - x^k> + D_\o(x,
x^k) \right\}, \\
	\grad (x^{k+1}) &= \grad(x^k) - \eta \nabla f(x^k).
\end{align*}
If one considers $\omega(x) = \frac{1}{2}\|x\|^2$, then mirror descent reduces to standard gradient descent. However, to obtain the Kullback-Leibler divergence, we need to use a specific entropy mapping, namely $\o(x) = x(\log x-1)$. Under this choice, it is easy to see that if $x,y \in R_+^d$, i.e.\ $x>0, y>0$, then
\begin{align*}
	D_\o(x, y) 
	&= \sum_{i=1}^d \left(x_i(\log x_i - 1)  - y_i(\log y_i - 1) - (\log y_i - 1)(x_i - y_i) \right) \\
	&= \sum_{i=1}^d (x_i\log \frac{x_i}{y_i} - x_i + y_i) \eqdef \KL(x||y).
\end{align*}
If $\sum_{i=1}^d x_i = \sum_{i=1}^d y_i=1$, then, $D_\o(x, y)=\sum_{i=1}^d x_i \log \frac{x_i}{y_i}$.

The analysis of mirror descent schemes usually assumes that $\omega$ is 1-strongly convex~\cite{beck2017first}, i.e.\ $D_\o(x, y)\ge \frac{1}{2}\|x-y\|^2$. The Kullback-Liebler divergence satisfies this assumption with respect to $\ell_1$ norm on a bounded set, but the objective that we will introduce is not constrained, so we can not assume it.

Note also that when $\o$ is strictly convex, we can express the mirror descent update using the convex conjugate. Indeed, define $\o^*(y)\eqdef \sup_{x}\{\<x, y> - f(x)\}$, then~\cite{beck2017first} $\grad^*(\grad(x))=x$ for any $x$, so $x^{k+1}=\nabla \o^*(\grad(x^k) - \gamma \nabla f(x^k))$.

\subsection{Bregman projections}
If, given $a\in \R^d, b\in\R$, we want to project $x$ onto the set $\{z:\<a, z>=b\}$ in Bregman divergence $D_\o$,  then we need to solve
\begin{align*}
	\Pi_C^\omega(x) \eqdef \argmin_{z: \<a, z>=b} \left\{\omega (z) - \omega(x) - \<\nabla \omega (x), z - x> \right\}.
\end{align*}
Although when $\o=\frac{1}{2}\|\cdot\|^2$ the projection can be computed in a closed form, this is not the case in general. The following lemma explains how one can solve the problem in other cases.
\begin{lemma}\label{lem:projection}
	The Bregman projection of $x\in\R^d$ onto the set $C=\{z:\<a, z>=b\}$ is given by $\Pi^\o_{C}(x)=\grad^*\left(\grad(x)+\alpha a\right)$, where $\alpha\in\R$ is the unique value satisfying $\<\grad^*(\grad(x)+\alpha a)), a>=b $. Furthermore, if all entries of $a$ are from $\{0, 1\}$ and $\omega(x)=x(\log x-1)$, then $\alpha=\log \frac{\<a, x>}{b}$ and $\Pi^\o_{C}(x)=x\odot \left(\frac{\<a, x>}{b}\right)^a$, where multiplication and exponentiation are coordinate-wise.
\end{lemma}
\begin{proof}
	Let $y \eqdef \Pi^\o_C(x)$. The necessary and sufficient optimality condition for $y$ to be the projection of $x$ is $\<\nabla \omega (y) - \nabla \omega(x), z - y>\ge 0$ for all $z: \<a, z>=b$. Take any $\delta$ such that $\<a, \delta> =0$, and plug in this optimality condition $z=y - \delta$ and $z=y + \delta$ to obtain
	\begin{align*}
		&\<\nabla \omega (y) - \nabla \omega(x), \delta>\ge0, 
		&\<\nabla \omega (y) - \nabla \omega(x), -\delta>\ge 0.
	\end{align*} 
	Therefore, $\<\nabla \omega (y) - \nabla \omega(x), \delta>=0$. Since it holds for arbitrary $\delta \perp a$, one has $\nabla \omega(y) = \nabla \omega(x) + \alpha a$ for some  $\alpha\in\R$. To find the exact value of $\alpha$, we need to solve $\<a, \nabla \omega^*(\nabla \omega(x) + \alpha a)> =b$. This equation is nonlinear in $\alpha$ and in general requires a one-dimensional  optimization problem to be solved~\cite{dhillon2007matrix}. However, if $\omega(x) = x(\log x-1)$ and all entries of $a$ are 0s and 1s, we obtain a simple equation,
	\begin{align*}
		\<a, \nabla \omega^*(\nabla \omega(x) + \alpha a)>
		=\sum_{i: a_i=1} \exp(\log(x_i) + \alpha)=b,
	\end{align*}
	 whence
%	\begin{align*}
		$\alpha = \log\frac{b}{\sum_{i: a_i=1} x_i} = \log \frac{b}{\<a, x>}$.
%	\end{align*}
\end{proof}
\subsection{Sinkhorn algorithm}
Sinkhorn algorithm uses the fact that Bregman projection onto $\{z:\<a, z>=b\}$, in accordance to Lemma~\ref{lem:projection}, changes vector $\grad(x^k)$ only in direction parallel to $a$. Therefore, if we run Bregman projections with sets $C_1, \dotsc, C_m$, where $C_i=\{z:\<a_i, z>=b_i\}$ for some $a_i\in\R^d, b_i\in \R$, by induction we have $\grad(x^k)=\grad(x^0)+\sum_{i=1}^m \alpha_i^k a_i$ with some $\alpha_1^k,\dotsc, \alpha_n^k\in\R$. 

Thus, if we apply Sinkhorn algorithm to problem~\eqref{eq:pb_reg} with two sets $C_1=\{X: X\1=p\}$, $C_2=\{X: X^\top \1=q\}$, which can be seen as intersections of $N$ coordinate-independent sets, it is enough to maintain just two vectors, $u^k, v^k\in\R^N$, one for each set. Using them, one can recover the matrix solution as $X^k=\diag(u^k)X^0\diag(v^k)$. Furthermore, two consecutive iteration get simplified to
\begin{align*}
	u^{k+1} = u^{k+1} = u^k + \log p - \log \left(X(u^k, v^k) \1\right), &&
	v^{k+2} = v^k + \log q - \log \left(X(u^{k+1}, v^k)^\top \1\right),
\end{align*}
where $X(u^k, v^k)\eqdef \diag(\exp(u^k))X^0\diag(\exp(v^k))$. To efficiently implement these updates, it is advised to use so-called log-sum-exp functions.
\section{Our results}
We now turn our attention to a new way of obtaining Sinkhorn algorithm. Namely, for the problem of finding a point $x^*> 0$ such that $Ax^*=b$ (assuming such $x^*$ exists) let us introduce the following objective,
\begin{align}
	\min_x f(x)\eqdef \KL(Ax|| b)\eqdef \sum_{i=1}^n f_i(x)=\sum_{i=1}^n\left(\<a_i, x>\log \frac{\<a_i, x>}{b_i} - \<a_i, x> + b_i\right). \label{eq:pb_ours}
\end{align}
Note that this is unconstrained minimization, although there an implicit constraint of having every ratio positive, i.e. $\frac{\<a_i, x>}{b_i}>0$ for all $i$. It is trivial to make it satisfied if $b>0$ and all entries of $a$ and $x$ are non-negative. The reformulation uses a penalty for each constraint, but unlike interior point methods~\cite{nesterov1994interior} it uses entropy function rather than log barrier. This objective is, in fact, not new and it has been considered in the literature on relative smoothness~\cite{bauschke2016descent, hanzely2018accelerated} without relating it to optimal transport.

The following theorem is the key finding of our work.
\begin{theorem}
	Let $f_i(x)\eqdef \KL(\<a_i, x>||b_i)= \<a_i, x>\left(\log \frac{\<a_i, x>}{b_i} - 1 \right) + b_i$, then, $\nabla f_i(x)=a_i \log \frac{\<a_i, x>}{b_i}$. Moreover, if all entries of $a_i\in \R^d$ are from $\{0,1\}$, then mirror step with respect to $f_i$  with stepsize $\eta=1$ and $\o(x)=x(\log x - 1)$ is equivalent to stochastic Bregman projection onto $C_i=\{z:\<a_i, z>=b_i\}$, i.e.
	\begin{empheq}[box=\mybluebox]{align*}
		\grad(x^{k+1}) = \grad(x^k) - \eta \nabla f_i(x^k) \Longleftrightarrow x^{k+1} = \Pi_{C_i}^\o(x^k).
	\end{empheq}
	Therefore, since $ \Pi_{C_i}^\o( \Pi_{C_i}^\o(x^k))= \Pi_{C_i}^\o(x^k)$, Sinkhorn algorithm is an instance of stochastic mirror descent.
\end{theorem}
\begin{proof}
	The equation for $\nabla f_i(x^k)$ follows by differentiation using the chain rule. The equivalence between the updates then follows by Lemma~\ref{lem:projection}. The statement about Sinkhorn algorithm follows from the fact that sampling a function twice will not change the iterate, so incremental and stochastic selection of the index lead to the same effective iterates.
\end{proof}
The equivalence between the methods holds only when solving problems similar to~\eqref{eq:pb_reg}. If we have a general matrix $A$, the new problem~\eqref{eq:pb_ours} is actually much more convenient since the gradients have closed-form expressions, unlike Bregman projections.

Let us also write a few properties that characterize the objective that we introduced.
\begin{enumerate}
	\item $\min_x f_i(x) = 0$ and $\argmin_x f_i(x)=\{z:\<a_i, x>=b_i\}$. 
	\item If the set $\{z: Az=b, z\ge 0\}$ is feasible, then there exists at least one $x^*\in\R^d$ such that $\nabla f_i(x^*)=0$ for all $i=1,\dotsc, n$. This property is called \textit{overparameterization}.
	\item $f_i$ is not smooth and does not have bounded gradients on $\dom f_i$.
	\item In the reformulation~\eqref{eq:pb_ours} of~\eqref{eq:pb_kl_proj}, $f_i$ is 1-\textit{relatively smooth} with respect to $\o(x)=x(\log x - 1)$~\cite{bauschke2016descent}, i.e.\ with constant $L_i=1$ it holds $D_{f_i}(x,y) \le L_i D_{f_i}(x, y)$ and $\nabla^2 f_i(x)\preceq L_i\nabla^2 \o(x)$ for any $x,y\in \dom\o$.
	\item Sinkhorn algorithm solves~\eqref{eq:pb_ours} with $n=2$, stepsize $\eta=1$ and it uses exactly one function at each iteration.
\end{enumerate}
In the penultimate statement we used the recently proposed notation of relative smoothness from~\cite{lu2018relatively, bauschke2016descent}. If we did not know about relative smoothness, it would be very surprising that Sinkhorn algorithm is convergent at all as the objective does not satisfy the regular assumptions of mirror descent. 
\subsection{New methods}
We see now that Sinkhorn is an instance of stochastic mirror descent applied to~\eqref{eq:pb_ours}. Furthermore,  one can mention that Greenkhorn~\cite{altschuler2017near, lin2019efficient} is greedy stochastic mirror descent. What if remove stochasticity from the method and just use gradient of $f$ as in~\cite{lu2018relatively}? This leads to a new method for optimal transport, which we call Pinkhorn (Penalty Sinkhorn) as the objective in~\eqref{eq:pb_ours} is a penalty reformulation similar to~\cite{nedic2010random, mishchenko2018stochastic}. Similarly, we can use accelerated mirror descent for relatively smooth optimization from~\cite{hanzely2018accelerated} to get accelerated Pinkhorn. Since for KL-divergence the method requires line-search and, ideally, restart procedure, we only provide the method in the appendix and refer the reader to~\cite{hanzely2018accelerated} for further implementation details.

On the other hand, one can instead generalize Sinkhorn algorithm as a stochastic mirror descent method for arbitrary linear systems beyond~\eqref{eq:pb_kl_proj}. There exists already a Stochastic Gradient Descent (SGD) method for relatively smooth minimization~\cite{hanzely2018fastest}, but it is not directly applicable since the variance in problem~\eqref{eq:pb_ours} is not bounded. Furthermore, it is not shown in~\cite{hanzely2018fastest} whether convergence will be to the projection of the initial iterate.

Finally, it is known~\cite{combettes2015solving} that the Bregman proximal operator $\prox_f^\o$ of $f(x)=\<x, \log x> - \<c, x>$ with distance function $\o(x)=x(\log x - 1)$  is given by
\begin{align*}
	\prox^\o_{\eta f(x)} \eqdef \argmin_{z}\left\{\eta f(z) + D_\o(z, x)\right\} = x\odot\exp\left(\frac{\eta(c - 1)}{\eta + 1}\right).
\end{align*}
This allows to invoke methods based on proximal operators for minimization of multiple functions such as~\cite{chambolle2016ergodic, mishchenko2019stochastic}. However, methods of this kind have not been extended to relatively smooth optimization, so we do not know yet whether they give an advantage. We leave this and other potential extensions of the obtained result for future work.

\clearpage

\bibliographystyle{plain}
\bibliography{sinkhorn_mirror}
\clearpage
\appendix
\part*{Supplementary Material}
\section{Full formulation of Pinkhorn}
\begin{algorithm}[h]
\caption{Pinkhorn (Penalty Sinkhorn).}
\label{alg:pinkhorn}
\begin{algorithmic}[1]
	\Require Stepsize $\eta\in (0,1) $, regularization $\gamma>0$, cost matrix $C$
	\State Initialize $u^0=0, v^0=0$, $X^0=\exp\left(-\frac{C}{\gamma}\right)$
	\For{$k = 0,1,2,\ldots$}
		\State $X(u^k, v^k)= \diag(\exp(u^k))X^0\diag(\exp(v^k))$
		\State $u^{k+1} = u^k + \eta\left(\log p - \log \left(X(u^k, v^k) \1\right)\right)$
		\State $v^{k+1} = v^k + \eta\left(\log q - \log \left(X(u^k, v^k)^\top \1\right)\right)$
	\EndFor
\end{algorithmic}	
\end{algorithm}
\begin{algorithm}[h]
\caption{Accelerated Pinkhorn, matrix form (Adaptation of Algorithm~4 from~\cite{hanzely2018accelerated})}
\label{alg:acc_pinkhorn}
\begin{algorithmic}[1]
	\Require Relative smoothness constant $L=2$, regularization $\gamma>0$, cost matrix $C$, line-search parameters $\rho>1$, $G_{\min}>0$
	\State Initialize $X^0=Z^0=\exp\left(-\frac{C}{\gamma}\right)$, $G_{-1}=1$, $\theta_0=1$
	\For{$k = 0,1,2,\ldots$}
		\State $M_k=\max\{G_{k-1}/\rho, G_{\min}\}$, $t=0$ (line-search counter)
		\Repeat
			\State $G_k=M_k\rho^t$
			\State \textbf{If} $k=0$, $\theta_0=1$, \textbf{otherwise} compute $\theta_k$ from $\frac{1-\theta_k}{G_k\theta_k^2}=\frac{1}{G_{k-1}\theta_{k-1}}$
			\State $Y^k = (1 - \theta_k)X^k + \theta_k Z^k$
			\State $\nabla f(Y^k)= (\log (Y^k\1) - \log p)\1^\top + \1(\log (\1^\top Y^k) - \log q^\top)$
			\State $\log Z^{k+1} = \log Z^k - \frac{1} {G_k\theta_k L} \nabla f(Y^k)$
			\State $X^{k+1} = (1 - \theta_k)X^k + \theta_k Z^{k+1}$
			\State $t\leftarrow t+1$
		\Until{{$f(X^{k+1})\le f(Y^k) + \<\nabla f(Y^k), X^{k+1} - Y^k> + G_k\theta_k LD_\o(Z^{k+1}, Z^k)$}}
	\EndFor
\end{algorithmic}	
\end{algorithm}
In this section we present full algorithmic formulation of Pinkhorn and its accelerated version. See Algorithm~\ref{alg:pinkhorn} and Algorithm~\ref{alg:acc_pinkhorn} for the details. The value of the objective function in Algorithm~\ref{alg:acc_pinkhorn} can be computed as
\begin{align*}
	f(X)&= \o(X\1, p) + \o(X^\top \1, q) \\
	&= \<X\1, \log (X\1/p)-1> + \<\1, p> + \<X^\top\1, \log(X^\top \1/q)-1> + \<\1, q>.
\end{align*}
\end{document}